\pdfoutput=1

\documentclass[11pt]{article}

\usepackage{acl}

\usepackage{times}
\usepackage{latexsym}

\usepackage[T1]{fontenc}

\usepackage[utf8]{inputenc}

\usepackage{microtype}

\usepackage{inconsolata}

\usepackage{graphicx}

\usepackage{booktabs}
\definecolor{darkblue}{rgb}{0, 0, 0.5}
\hypersetup{colorlinks=true, citecolor=darkblue, linkcolor=darkblue, urlcolor=darkblue}

\usepackage{amsthm}
\newtheorem{theorem}{Theorem}
\newtheorem{lemma}[theorem]{Lemma}

\usepackage{multirow}

\usepackage{subcaption}
\usepackage{wrapfig}
\usepackage{longtable}

\usepackage{amsmath,amsfonts,bm}

\def\Secref#1{Section~\ref{#1}}

\def\eqref#1{equation~\ref{#1}}

\def\1{\bm{1}}
\newcommand{\train}{\mathcal{D}}

\DeclareMathAlphabet{\mathsfit}{\encodingdefault}{\sfdefault}{m}{sl}
\SetMathAlphabet{\mathsfit}{bold}{\encodingdefault}{\sfdefault}{bx}{n}

\newcommand{\E}{\mathbb{E}}

\newcommand{\sigmoid}{\sigma}

\title{Binary Classifier Optimization for Large Language Model Alignment}

\author{
  \textbf{Seungjae Jung\textsuperscript{$\diamondsuit$}} \quad
  \textbf{Gunsoo Han\textsuperscript{$\diamondsuit$}} \quad
  \textbf{Daniel Wontae Nam\textsuperscript{$\diamondsuit$}} \quad
  \textbf{Kyoung-Woon On\textsuperscript{$\clubsuit$}}
  \thanks{Work done at Kakao Corp.}
  \\
  \\
  \textsuperscript{$\diamondsuit$}Kakao Corp\\
  \textsuperscript{$\clubsuit$}LBOX \\
  \texttt{\{sean.ai, coco.upgrade, daniel.rl\}@kakaocorp.com} \\
  \texttt{kyoungwoon.on@lbox.kr} \\
  \\
}

\begin{document}
\maketitle
\begin{abstract}
In real-world services such as ChatGPT, aligning models based on user feedback is crucial for improving model performance.
However, due to the simplicity and convenience of providing feedback, users typically offer only basic binary signals, such as 'thumbs-up' or 'thumbs-down'.
Most existing alignment research, on the other hand, relies on preference-based approaches that require both positive and negative responses as a pair.
We propose Binary Classifier Optimization (BCO), a technique that effectively aligns LLMs using only binary feedback.
BCO trains a binary classifier, where the logit serves as an implicit reward, effectively minimizing the Direct Preference Optimization (DPO) loss.
We demonstrate that the binary cross-entropy loss employed in classifier training acts as an upper bound for the DPO loss.
Additionally, a novel reward shift technique further minimizes the gap between the losses.
We validate our methodology in two settings: first, on a paired preference dataset, where our method performs on par with DPO; and second, on a Likert-5 scale annotation dataset which stems from real users' queries.
Our model consistently demonstrates effective and robust alignment across four base LLMs and three different datasets, showcasing the strength of our approach to learning from binary signals.
\end{abstract}

\section{Introduction}

Aligning Large Language Models (LLMs) has been a crucial ingredient in the deployment of LLMs in production, as pretrained LLMs are prone to generating undesirable outputs. 
\citet{ouyang2022training} introduced Reinforcement Learning with Human Feedback (RLHF), that involves training a reward model based on various completions and their comparisons for a single prompt and then optimizing the LLM to maximize those rewards.
Subsequently, Direct Preference Optimization (DPO) \citep{rafailov2023direct} was proposed as an alternative that circumvents the need for training a reward model by directly optimizing the model based on the preferences between chosen and rejected completions.
Both RLHF and DPO have emerged as the standard choices for LLM alignment, but they still require a comparison dataset with chosen and rejected text completions, which is labor-intensive to collect.

In reality, when it comes to serving LLMs to users, it is much easier to collect binary feedback rather than comparison between two completions.
Popular LLM services such as ChatGPT \citep{ChatGPT}, Gemini \citep{Gemini}, or Claude \citep{Claude} ask users for "thumbs-up" or "thumbs-down" feedbacks.
On the other hand, most existing alignment research relies on preference-base methodologies that require at least two responses and their relative goodness.

Counter to this trend, a recent work called Kahneman-Tversky Optimization (KTO) has been proposed \citep{ethayarajh2024kto}.
KTO, inspired by the Prospect Theory \citep{tversky1992advances} in economics, offers a promising approach to alignment that requires only a single completion per prompt, accompanied by a binary signal of preference, such as a "thumbs-up" or "thumbs-down".
This development increases the possibility of eliminating the laborious process of comparing completions to create preference datasets, making the alignment process more agile and accessible.

Nevertheless, the theoretical connection between alignment from binary signals and DPO has not been thoroughly explored.
Understanding this connection could provide opportunities to further enhance the performance of alignment from binary signals.

In this paper, we present a theoretical foundation for the efficacy of alignment from the binary signals as a binary classification problem.
Our analysis reveals that training a binary classifier, where the logit serves as a reward, effectively maps \{prompt, thumbs-up completion\} pairs to 1 and \{prompt, thumbs-down completion\} pairs to 0, implicitly minimizes the DPO loss.
Specifically, the binary cross-entropy (BCE) loss used in the classifier training serves as an upper bound for minimizing the DPO loss.
Furthermore, we devise a novel reward shift technique that further decreases the gap between the BCE loss and the DPO loss, leading to improved alignment.
Our analysis theoretically and empirically uncovers potential flaws in the reference point used in KTO that can be rectified using our reward shift technique.
Integrating the reward shift technique to the BCE loss, we propose a novel framework for aligning language models using binary signals which we name Binary Classifier Optimization (BCO).

We validate our methodology in two type of datasets: paired preference dataset and real-world Likert-5 scale annotation dataset.
On the paired preference datasets we demonstrate that our method surpasses KTO and performs on par with DPO.
On the real-world Likert-5 scale annotation dataset, empirical results confirm the superiority of BCO over DPO and KTO across four configurations of base LLMs, including Qwen and Llama \citep{qwen2.5, dubey2024llama}, in both small and medium model sizes.

\section{Related Work}

Reinforcement Learning from Human Feedback (RLHF) \citep{ouyang2022training, stiennon2020learning, glaese2022improving, ziegler2019fine} has garnered significant attention as a promising approach for aligning LLMs with human preferences.
While RLHF is effective, it is burdensome as it requires going through three stages: supervised fine-tuning (SFT), reward modeling, and reinforcement learning (RL).
The RL stage is particularly memory-intensive, as it requires loading the policy, reference, reward model, and value function into memory.
The introduction of DPO \citep{rafailov2023direct} has improved the accessibility of LLM alignment by eliminating the reward modeling stage.
DPO directly optimizes the policy to satisfy human preferences using a loss function derived from the Bradley-Terry (BT) model \citep{bradley1952rank}.

One potential drawback of DPO is its susceptibility to overfitting the preference dataset.
To address this issue, Identity Preference Optimization (IPO) \citep{azar2023general} introduces a regularization term to mitigate overfitting.
Rejection Sampling Optimization \citep{liu2023statistical} employs rejection sampling to generate preference pairs from the estimated optimal policy.
Although these methodologies share commonalities with our work, as they offer theoretical insights into the BT model and propose enhanced alignment approaches, they still depend on preference datasets, which sets them apart from our work.

To reduce the effort required to collect preference datasets, methodologies have been proposed that either let the LLM itself perform comparisons of completions \citep{yuan2024self}
or treat the LLM's completions as rejected completions \citep{chen2024self} .
However, none of them utilized binary signals for LLM alignment.

In contrast, KTO \citep{ethayarajh2024kto}, which is inspired by prospect theory \citep{tversky1992advances}, is designed to align LLMs using only thumbs-up and thumbs-down datasets without the need to construct a preference dataset.
In terms of aligning LLMs from binary signals, KTO is the most similar work to ours.
Unlike KTO, we theoretically demonstrate the connection between alignment from binary signals and preference optimization.
Based on this, we present an effective algorithm for robust alignment in real-world scenarios. The detailed differences between our approach, BCO, and KTO are illustrated in \Secref{subsec:distinctions}.

\citet{chen2024noise} proposed Noise Contrastive Alignment (NCA), which enables alignment from explicit rewards.
While NCA allows alignment from binary signals, it requires multiple completions per prompt, differing from BCO/KTO in the scope of problems it can address.
The distinctions between our approach, BCO, and NCA are further elaborated in \autoref{subsec:distinctions}.

\section{Preliminaries} \label{sec:preliminaries}

Aligning LLMs to human preference follows a widely adopted convention from \citet{ouyang2022training}, consisting of three main stages: SFT, reward modelling, and RL.
During SFT, given an input prompt $x$ and an corresponding completion $y$ from the dataset $\train$, the generation probability of $y$ given $x$ i.e. $-\mathbb{E}_{(x, y) \sim \mathcal{D}} \left[ \log p(y \vert x) \right]$ is maximized.
During the reward modelling stage, a separate reward model is trained to assign appropriate scalar rewards that reflect human preference to given \{ prompt, completion \} pairs. 
Finally, RL is applied to further align the model gained from SFT, which typically involves optimizing a policy using the obtained reward model. 

In the RL stage, it is a common practice to incorporate a regularization term that encourages the policy to remain close to the reference model \citep{ziegler2019fine, ouyang2022training}:
\begin{align} \label{eq:rl_obj}
    \mathbb{E}_{(x,y) \sim \mathcal{D}} \left[ r(x,y)\right] - \beta \text{KL}(\pi_\theta(\cdot \mid x) \Vert \pi_{\text{ref}}(\cdot \mid x))
\end{align}

\paragraph{DPO}
While RLHF with trained reward model has been shown to be successful, it yields challenges such as large computational burden and requires an additional training phase.
DPO \citep{rafailov2023direct} demonstrated a clever solution to circumvent the challenges by showing that the policy $\pi_\theta$ can be directly optimized using the preference dataset $\mathcal{D}$ by using the reward-policy relationship derived from \autoref{eq:rl_obj}.
The implicit reward function can be defined as a function of the policy such that $r_\theta(x,y) = \beta \log \frac{\pi_\theta(y \vert x)}{\pi_\text{ref}(y \vert x)}$ without losing generality in the theoretical foundation behind DPO.
Combining the BT model with the reward model, the loss function of DPO is
\begin{align*}
    -\mathbb{E}_{(x,y_w,y_l) \sim \mathcal{D}} \left [ \log \sigma \left( r_\theta(x,y_w) - r_\theta(x,y_l)\right) \right].
\end{align*}
Here, $y_w$ is a chosen completion and $y_l$ is a rejected completion.

\paragraph{KTO} \label{paragraph:kto}
\citet{ethayarajh2024kto} proposed alignment framework that trains on binary signal of thumbs-up or thumbs-down of a completion per prompt.
Given a dataset of \{ prompt, completion \} pairs with respective binary signals, KTO defines a value function
\begin{align} \label{eq:kto_value}
&v_{KTO}(x, y; \theta) \nonumber \\
&= \begin{cases}
\sigma(r_\theta(x, y) - z_\text{ref}) & \text{if } y \sim y_{\text{desirable}} \mid x \\
\sigma(z_\text{ref} - r_\theta(x, y)) & \text{if } y \sim y_{\text{undesirable}} \mid x,
\end{cases}
\end{align}
where $z_\text{ref}$ is a reference point.
In practice, $z_\text{ref}$ is implemented as
\begin{equation} \label{eq:kto_z_ref}
z_{\text{ref}} = \max \left( 0, \frac{1}{|\mathcal{B}|} \sum_{y' \in \mathcal{B} \setminus y} \log \frac{\pi_\theta(y' \vert x)}{\pi_\text{ref}(y' \vert x)} \right)
\end{equation}
for $(x,y) \in \mathcal{B}$ and $\mathcal{B}=\{(x^{(i)}, y^{(i)})\}_{i=1}^B$ is a batch of samples.

Finally, the loss function of KTO is defined as 
\begin{equation} \label{eq:kto_loss}
\mathcal{L}_{\text{KTO}}(\theta) = \mathbb{E}_{(x,y) \sim \mathcal{D}} \left [ w(y) (1 - v_\text{KTO}(x,y; \theta) \right]
\end{equation}
where the weighting factor $w(y)$ is $\lambda_D$ if $y$ is a completion from thumbs-up dataset and $\lambda_U$ if $y$ is a completion from thumbs-down dataset.


\section{Binary Classifier Optimization} \label{sec:binary_classifier_optimization}

In this section, we explore the theoretical foundation that could explain the effectiveness of aligning LLMs using binary signals, which are much easier to collect than pairwise preference datasets.
We propose Binary Classifier Optimization (BCO), a novel approach that achieves robust alignment from binary signals upon the theoretical foundation. 

Throughout the section, we illustrate alignment process in terms of optimizing reward.
It is important to note that implicit reward optimization is sufficient for alignment due to the reward-policy relationship
\begin{equation*}
r_\theta (x, y) = \beta \log \frac{\pi_\theta (y \mid x)}{\pi_{\text{ref}} (y \mid x)}
\end{equation*}
which already has been shown both theoretically and empirically in previous works \citep{rafailov2023direct, azar2023general, ethayarajh2024kto, chen2024noise}.

\subsection{Theoretical Analysis}

For simplicity, let's momentarily assume that $z_{\text{ref}}$ is 0 in \autoref{eq:kto_value}.
As mentioned in \autoref{sec:preliminaries}, the DPO loss minimizes $- \log \sigmoid(r_\theta(x, y_w) - r_\theta(x, y_l))$, while the KTO loss minimizes $- \sigmoid(r_\theta(x, y_w)) - \sigmoid(-r_\theta(x, y_l))$.
By establishing a connection between the two terms, we can bridge the gap between DPO and alignment from binary signals.

\begin{theorem} \label{thm:upper_bound}
For a binary classifier that assigns a reward logit, where \{ prompt, chosen completion \} pairs are mapped to 1 and \{ prompt, rejected completion \} pairs are mapped to 0, minimizing the binary cross-entropy loss between the true and predicted labels upper bounds the direct preference optimization loss. i.e.
\begin{align*}
&\E_{(x, y_w, y_l) \sim \train} [-\log \sigmoid \left( r_\theta (x, y_w) - r_\theta(x, y_l) \right) ] \\
&< 
\E_{(x, y_w, y_l) \sim \train} [- \log \sigmoid(r_\theta(x, y_w))] \\
& \qquad + \E_{(x, y_w, y_l) \sim \train} [- \log \sigmoid \left( - r_\theta (x, y_l) \right)]
\end{align*}
\end{theorem}

To prove the above theorem, we prove the lemma below.

\begin{lemma} \label{lem:log_bound}
The log of sigmoid of a sum exceeds the sum of the logs of the sigmoids. i.e. $\log \sigmoid (x + y) > \log \sigmoid(x) + \log \sigmoid(y)$ for all $x, y \in \mathbb{R}$ 
\end{lemma}

See \autoref{subsec:sigmoid_lemma_proof} for the proof.
Simply applying Lemma \ref{lem:log_bound} and linearity of expectation concludes the proof of \autoref{thm:upper_bound}.
\begin{align}
&\E_{(x, y_w, y_l) \sim \train} [-\log \sigmoid \left( r_\theta (x, y_w) - r_\theta(x, y_l) \right) ] \nonumber \\
&< \E_{(x, y_w, y_l) \sim \train} [- \log \sigmoid(r_\theta(x, y_w)) \nonumber \\
& \qquad \qquad \qquad - \log \sigmoid(- r_\theta(x, y_l))] \label{eq:before_expectation_linearity} \\
&= \E_{(x, y_w, y_l) \sim \train} [- \log \sigmoid(r_\theta(x, y_w))] \nonumber \\
& \qquad + \E_{(x, y_w, y_l) \sim \train} [- \log \sigmoid(- r_\theta(x, y_l))] \label{eq:binary_cross_entropy}
\end{align}
Equation \ref{eq:binary_cross_entropy} is the binary cross-entropy (BCE) loss, where the logit of the binary classifier is the reward implicitly defined by the policy and reference models.
Since the BCE loss serves an upper bound of the DPO loss, LLM alignment can be performed using only binary signals.

According to \autoref{eq:upper_bound} in \autoref{subsec:sigmoid_lemma_proof}, the tightness of the BCE loss as a bound for the DPO loss depends on the error term $e^{-x} + e^{-y}$ where $x = r_\theta(x, y_w)$ and $y = - r_\theta(x, y_l)$.
As training progresses and the BCE loss is minimized, the magnitude of $r_\theta(x, y_w)$ increases while the magnitude of $r_\theta(x, y_l)$ decreases, leading to decrease of the error term.
Consequently, the BCE loss becomes a tighter bound for the DPO loss. 
Empirical evidence presented in \autoref{sec:experiments} demonstrates that, despite the presence of an error term, alignment progresses solely with the BCE loss.

\subsection{Reward Shift} \label{subsec:reward_shift}

We further minimize the error term $e^{-r_\theta(x, y_w)} + e^{r_\theta(x, y_l)}$ by reward shift.

Consider the case where the reward is shifted by $\delta$ in \autoref{eq:before_expectation_linearity}.
That says,
\begin{align*}
\E_{(x, y_w, y_l) \sim \train} [
&- \log \sigmoid(r_\theta(x, y_w) - \delta) \\
&- \log \sigmoid(- (r_\theta(x, y_l) - \delta))]
\end{align*}
The binary cross-entropy loss still holds as an upper bound of the DPO loss. 

\begin{theorem} \label{thm:upper_bound_hold_under_shift}
Binary cross entropy is an upper bound of Direct Preference Optimization loss even if the reward is shifted by a constant $\delta$. i.e.
\begin{align*}
&\E_{(x, y_w, y_l) \sim \train} [-\log \sigmoid \left( r_\theta (x, y_w) - r_\theta(x, y_l) \right) ] \\
&< \E_{(x, y_w, y_l) \sim \train} [- \log \sigmoid(r_\theta(x, y_w) - \delta) \\
& \qquad \qquad - \log \sigmoid(- (r_\theta(x, y_l) - \delta))]
\end{align*}
\end{theorem}
See \autoref{subsec:upper_bound_proof} for the proof.
Expanding the inside of the expectation as in the proof of Lemma \ref{lem:log_bound} in \autoref{subsec:sigmoid_lemma_proof}, we get the error term
\begin{equation*}
e^{-(r_\theta(x, y_w) - \delta)} + e^{r_\theta(x, y_l) - \delta}  
\end{equation*}

Setting appropriate $\delta$ minimizes the error term, leading to closer gap between the BCE loss and the DPO loss.
\begin{theorem} \label{thm:optimal_delta}
The minimum of the error term $e^{-(r_\theta(x, y_w) - \delta)} + e^{r_\theta(x, y_l) - \delta}$ can be achieved when $\delta = (r_\theta(x, y_w) + r_\theta(x, y_l)) / 2$
\end{theorem}

See \autoref{subsec:optimal_delta_proof} for the proof.
Hence, for alignment using binary signals, we define $\delta$ as follows:
\begin{equation} \label{eq:bco_delta}
\delta = \frac{\E_{(x, y) \sim \train^+} [r_\theta(x, y)] + \E_{(x, y) \sim \train^-} [r_\theta(x, y)]}{2}
\end{equation}
Here, $\train^+$ and $\train^-$ denote thumbs-up and thumbs-down datasets of prompt-completion pairs respectively.
Consequently, the BCO loss for a binary signal dataset can be expressed as:
\begin{align}
&\E_{(x, y_w, y_l) \sim \train^+} [- \log \sigmoid(r_\theta(x, y_w) - \delta)] \nonumber \\
& \qquad + \E_{(x, y_w, y_l) \sim \train^-} [- \log \sigmoid(- (r_\theta(x, y_l) - \delta))] \label{eq:bco_objective}
\end{align}
To enhance training stability, we utilize an exponential moving average when computing $\delta$.
The efficacy of this reward shift approach is empirically demonstrated in \autoref{sec:experiments}.

\subsection{Distinctions from Prior Works} \label{subsec:distinctions}

So far, we delved into the connection between BCO and DPO, demonstrating BCO's applicability to alignment from binary signal scenarios.
This subsection delineates the key distinctions between BCO and variants of DPO.

KTO is the first DPO variant we will contrast with BCO.
Both algorithms are quite similar in that they enable alignment from binary signals, meaning they can learn even when only one completion is provided for a single prompt along with user feedback.
However, despite the similarity, there are two critical distinctions between the two algorithms.

While BCO objective in \autoref{eq:bco_objective} optimizes the logsigmoid, KTO objective in \autoref{eq:kto_loss} optimizes the sigmoid. This distinction becomes more apparent when differentiating the objectives.
For simplicity of analysis, assume $z_{ref}$ and $\delta$ are both zero. 
\begin{align*}
\nabla_\theta \mathcal{L}_{\text{BCO}} &= \mathbb{E}_{x, y \sim \train} [\sigmoid(-r_\theta) \nabla_\theta \beta \log \pi_\theta (y \mid x)] \\
\nabla_\theta \mathcal{L}_{\text{KTO}} &= \mathbb{E}_{x, y \sim \train} [\sigmoid (r_\theta) \sigmoid(-r_\theta) \nabla_\theta \beta \log \pi_\theta (y \mid x)]
\end{align*}

Here, $r_\theta = r_\theta (x, y)$.
For brevity, we derive the gradients for the case where $y \sim y_{\text{desirable}}$.
The difference between the gradients of the two algorithms depends on the presence of the sigmoid term $\sigmoid(r_\theta(x, y))$.
In KTO, $\sigmoid(r_\theta(x, y))$ causes samples $(x, y)$ with low rewards to be learned less, whereas BCO does not vanish the gradients for such low-reward samples.
A similar analysis can be conducted for $y \sim y_{\text{undesirable}}$, where BCO better preserves the gradients for high-reward $(x, y)$ samples.
In brief, BCO should be employed if one wishes to treat all data samples equitably.

BCO and KTO also differ in their reward shifting approach.
BCO takes the average implicit reward of $(x, y)$ as the reference point, while KTO adopts the average reward of $(x, y^\prime)$, where $y^\prime$ is a unrelated completion of $x$, as the reference point.
Notably, KTO's reference point is clipped at zero to ensure it remains positive.
Ultimately, this zero clipping hinders seamless model training.
According to the KTO loss, for $y \sim y_{\text{desirable}}$, the implicit reward is increased relative to the reference point, and for $y \sim y_{\text{undesirable}}$, it is decreased relative to the reference point.
Consequently, the average implicit reward remains anchored at the reference point.
However, as pointed out by \citet{rafailov2024from}, the average implicit reward is equivalent to $- \beta \text{KL} (\pi_{\text{ref}} (\cdot \mid x) \Vert \pi_\theta (\cdot \mid x))$ \footnote{See \autoref{sec:average_implicit_reward} for detailed explanation of why average implicit reward is equivalent to KL}, which needs to decrease.
Otherwise, $\pi_\theta$ stay too close to $\pi_{\text{ref}}$ and will not effective learn from preference data.
Therefore, KTO's reference point zero clipping obstructs training, as elaborated in \autoref{subsec:effect_of_reward_shift}.
In contrast, BCO avoids this issue by setting the reference point as the average implicit reward without artificial clipping.

The second DPO variant to contrast with BCO is NCA \citep{chen2024noise}.
When learning from a preference dataset, NCA's loss takes the following form:
\begin{equation*}
- \log \sigmoid (r_\theta (x, y_w)) - \frac{1}{2} \sum_{y \in \{ y_w, y_l \}} \log \sigmoid (r_\theta (x, y))
\end{equation*}
The presence of $\log \sigmoid(r_\theta(x, y))$ in the loss bears similarity to BCO's loss.
However, as evident from the latter term of the objective, computing the partition function is required, necessitating multiple completions for a given prompt.
Consequently, direct alignment from user feedback is infeasible, limiting the scope of problems NCA can address compared to BCO.

\section{Experiments} \label{sec:experiments}

In this section, we compare BCO with offline preference tuning methods. \footnote{As recent works \citep{xu2024is, tang2024understanding} have revealed that online methods outperform offline methods, we do not include PPO \citep{Schulman2017ProximalPO} as compared methods.}
To investigate the effect of reward shift, we augment the compared methods with BCE, where $\delta$ is set to 0 in the BCO objective in \autoref{eq:bco_objective}.
We aim to answer three key research questions:
1) Does the simple BCE loss fuses alignment capability to LLMs?
2) Does the proposed reward shift technique contribute to the alignment process?
3) What is the advantage of BCO over DPO?

\subsection{Experimental Setup}
\label{subsec:experimental_setup}

\paragraph{Dataset}
\label{para:dataset}

We utilize three publicly available preference datasets: UltraFeedback\footnote{https://huggingface.co/datasets/trl-lib/ultrafeedback\_binarized} \citep{cui2023ultrafeedback}, Capybara\footnote{https://huggingface.co/datasets/trl-lib/Capybara-Preferences} \citep{daniele2023amplify-instruct}, and HelpSteer2 \citep{wang2024helpsteer2}.
UltraFeedback and Capybara provide sets of chosen and rejected responses for each prompt.
The HelpSteer2 dataset includes prompts, completions, and various metrics, such as a helpfulness score.
Each prompt is associated with two alternative completions, enabling its conversion into a paired preference dataset.

\paragraph{Model}
\label{para:model}

Our experiments involve four model classes: \texttt{Llama-3.2-3B}, \texttt{Llama-3.1-8B} \citep{dubey2024llama}, \texttt{Qwen2.5-3B}, and \texttt{Qwen2.5-7B} \citep{qwen2.5}.
Unless specified otherwise, we initially conduct Supervised Fine-Tuning (SFT) using the respective datasets.
The chosen response is used as the SFT target as it is recommended by \citet{rafailov2023direct}.
Detailed training specifications are available in \autoref{app:implement}.
We maintain consistent hyperparameters across all experiments, with the exception of the number of training epochs.
Furthermore, for all experiments evaluating win rate, \texttt{gpt-4o-2024-08-06} serves as the evaluation judge.



\subsection{Experiments on the Preference Dataset}
\label{subsec:preference_dataset_experiments}

As illustrated in \autoref{fig:pref_combined}, the performance of KTO surpasses that of SFT, yet it generally falls short of DPO across most configurations.
Similarly, employing a basic BCE loss results in diminished performance when compared to DPO.
Nonetheless, it is important to note that the simple BCE loss consistently outperforms the SFT model in all instances, suggesting that BCE loss contributes to enhancing the alignment capability of LLMs.
On the other hand, we observe a notable improvement in performance when applying reward shift compared to BCE.
This enhancement, coupled with a reduction in error terms, empirically underscores the beneficial impact of reward shifts, as outlined in \autoref{subsec:reward_shift}.
In most scenarios, BCO achieves performance levels comparable to DPO.
While BCO shows superior outcomes over DPO in training models such as \texttt{Llama-3.1-8B} and \texttt{Qwen2.5-7B} with the UltraFeedback dataset, the discrepancy in their performance is not statistically significant.

\begin{figure}[t]
    \centering
    \begin{subfigure}{0.46\textwidth}
        \centering
        \includegraphics[width=\textwidth]{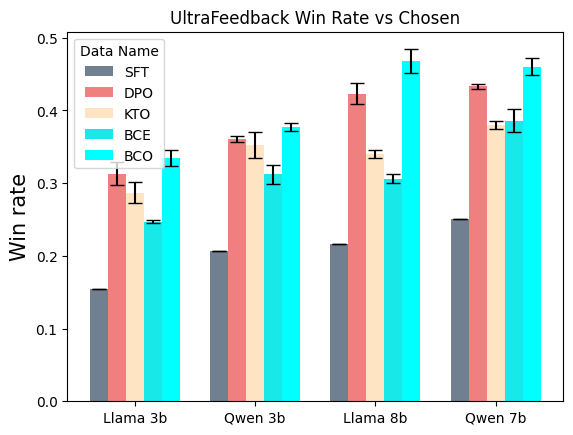}
        \caption{UltraFeedback}
        \label{fig:pref_ultrafeedback}
    \end{subfigure}
    \begin{subfigure}{0.46\textwidth}
        \centering
        \includegraphics[width=\textwidth]{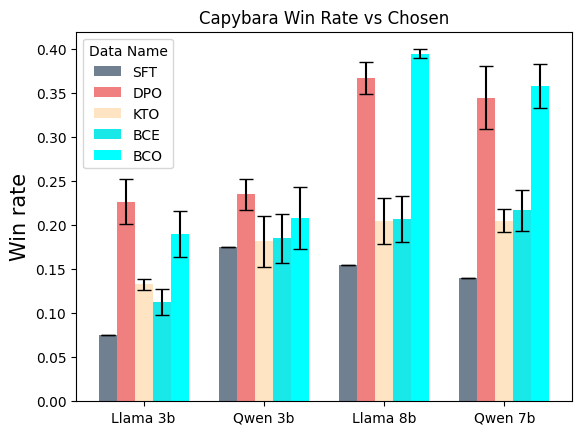}
        \caption{Capybara}
        \label{fig:pref_capybara}
    \end{subfigure}
    \caption{
    Win rates computed by GPT-4o on UltraFeedback and Capybara datasets.
    The win rates are calculated against chosen completions in the test set.
    Depicted mean and standard deviation of the win rates are obtained from three different random seeds.
    }
    \label{fig:pref_combined}
\end{figure}

\subsection{Experiments on the Likert-5 Scale Dataset}

To demonstrate the superiority of BCO over DPO, we present experimental results using a dataset with Likert-5 scale feedback.
We select the HelpSteer2 dataset \citep{wang2024helpsteer2} for alignment purposes for two main reasons.
First, its reward model, trained using the \texttt{Llama-3-70B} base model \citep{dubey2024llama}, demonstrated exceptional performance in the RewardBench benchmark \citep{lambert2024rewardbench}.
Second, the dataset closely resembles real-world data, as most of its prompts originate from ShareGPT \citep{sharegpt}.
To facilitate DPO training, we transformed the HelpSteer2 training set into a preference dataset following the methodology outlined by \citep{wang2024helpsteer2}.
In the HelpSteer2 dataset, each prompt is paired with two completions that are assigned helpfulness scores.
The response with the higher helpfulness score is designated as the preferred choice, while the other is considered as the rejected response.
Pairs with identical helpfulness scores were excluded from this process.

To facilitate the training for both BCO and KTO, we convert HelpSteer2 dataset into a binary signal dataset.
In this conversion, a helpfulness score of 4 is mapped to a thumbs-up, while scores of 3 or below are mapped to a thumbs-down.
To ensure a fair comparison with DPO, only the prompts used in DPO training are included in the binary signal dataset.
See \autoref{tab:stat_helpsteer2} for statistics after processing and \autoref{sec:qualititive_results} for the generated response of each methodology.

\begin{figure}[t]
  \includegraphics[width=0.96\columnwidth]{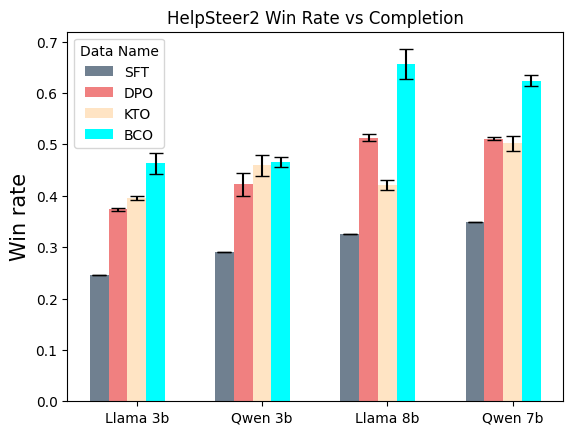}
  \caption{
    Win rates computed by GPT-4o for HelpSteer2 dataset.
    The win rates are calculated against completions in the test set.
    Depicted mean and standard deviation of the win rates are obtained from three different random seeds.
  }
  \label{fig:likert_helpsteer}
\end{figure}

As shown in \autoref{fig:likert_helpsteer}, KTO outperforms DPO only in small-sized models.
In contrast, BCO outperforms DPO across all models.
In summary, \autoref{fig:likert_helpsteer} illustrates that, for the purpose of model alignment, converting a Likert-5 scale dataset directly into a binary signal dataset is not only feasible but may also yield superior performance.

\begin{table*}[t]
    \centering
    \resizebox{\linewidth}{!}{
    \begin{tabular}{|l|l|l|c|c|}
        \hline
        & MT Bench & AlpacaEval & Arena-Hard & Length \\ \midrule
        Llama-3.1-8B-Instruct & 8.28 & 20.9 & 26.86 (-2.0, 2.0) & 830\\
        + DPO \citep{rafailov2023direct} & 8.17 $\pm$ 0.10 & 25.99 $\pm$ 0.48 & 29.50 (-2.3, 1.9) & 746 \\ 
        + IPO \citep{azar2023general} & 8.19 $\pm$ 0.05 & \textbf{30.31 $\pm$ 0.75} & 20.39 (-1.9, 1.9) & \textbf{432} \\
        + KTO \citep{ethayarajh2024kto} & 8.24 $\pm$ 0.10 & 23.42 $\pm$ 0.81 & 27.81 (-1.8, 2.1) &  707 \\ 
        + NCA \citep{chen2024noise} & \textbf{8.32 $\pm$ 0.14} & 25.63 $\pm$ 0.18 & 29.64 (-2.3, 2.0) & 728 \\
        + BCO (Ours) & \textbf{8.32 $\pm$ 0.10} & 28.61 $\pm$ 0.21 & \textbf{31.37} (-2.1, 2.2) & 762 \\ \midrule
        Qwen2.5-7B-Instruct & 8.43 & 31.43 & 47.73 (-2.4, 2.3) & 776 \\
        + DPO & 8.40 $\pm$ 0.10 & 30.95 $\pm$ 1.28 & 47.82 (-2.5, 2.5) & 772 \\ 
        + IPO & 8.48 $\pm$ 0.14 & \textbf{32.97 $\pm$ 1.47} & 50.18 (-2.6, 2.3)  & \textbf{636} \\
        + KTO & 8.30 $\pm$ 0.07 & 28.25 $\pm$ 0.68 & 47.76 (-1.9, 2.6) &  775 \\
        + NCA & 8.45 $\pm$ 0.16 & 30.24 $\pm$ 0.70 & 48.98 (-2.7, 2.3) & 791 \\
        + BCO (Ours) & \textbf{8.59 $\pm$ 0.04} & 30.54 $\pm$ 0.53 & \textbf{50.60} (-2.1, 2.2) & 754 \\ \hline
    \end{tabular}
    }
    \caption{
    Alignment benchmark results for models are presented.
    The alignment training is conducted on the \texttt{Llama-3.1-8B-Instruct} and \texttt{Qwen2.5-7B-Instruct} models.
    All models are trained using HelpSteer2 dataset.
    For the MT Bench and AlpacaEval 2.0 Length Controlled, the mean and standard deviations across three different random seeds are reported.
    For the reference models, we conduct only a single evaluation, so the standard deviations are set to zero.
    For the Arena-Hard benchmark, the win rate against the GPT-4-0314 model, along with the confidence intervals, is provided.
    The length column indicates the average number of tokens generated in the Arena-Hard benchmark.
    }
    \label{tab:alignment_benchmark_results}
\end{table*}

\subsection{Evaluation on Chat Benchmarks}

To further validate the superiority of BCO on well-known alignment benchmarks, we measure the MT Bench \citep{zheng2023judging}, AlpacaEval 2.0 Length Controlled (LC) \citep{dubois2024lengthcontrolled}, and Arena-Hard \citep{li2024arenahard} scores of models.
All models are trained using HelpSteer2 dataset.
\autoref{tab:alignment_benchmark_results} presents the benchmark performance results after applying alignment methods, using \texttt{Llama-3.1-8B-Instruct} and \texttt{Qwen2.5-7B-Instruct} as the reference models.

Except for the AlpacaEval 2.0 LC performance, BCO outperforms other methodologies.
For AlpacaEval 2.0 LC performance, we observe that only IPO clearly outperforms BCO.
Additionally, it is encouraging that, in the Arena-Hard benchmark, BCO demonstrates superior performance despite having a generated token length similar to that of DPO.

\subsection{Effect of Reward Shift} \label{subsec:effect_of_reward_shift}

\begin{figure}[ht]
  \includegraphics[width=0.9\columnwidth]{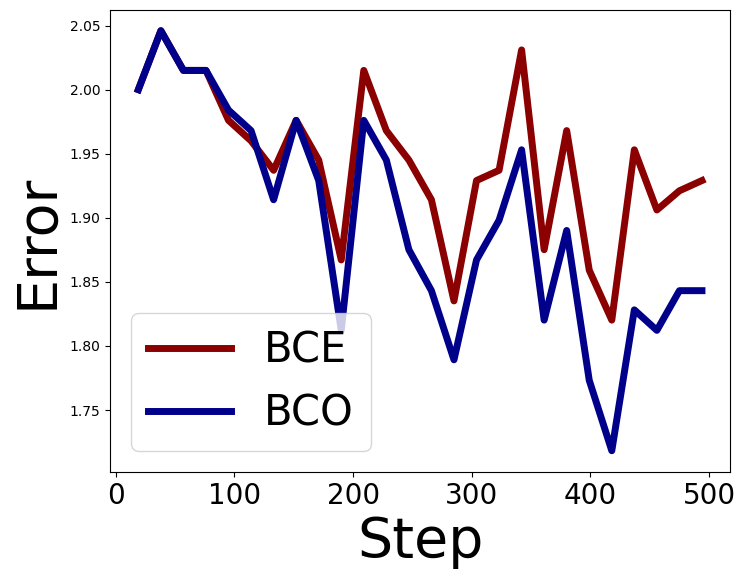}
  \caption{
    Error term values per step on the UltraFeedback dataset are presented.
    These values are derived from the training of the SFT variant of the \texttt{Llama-3.2-3B} model.
    Note that the only difference between BCE and BCO is the existence of $\delta$ in \autoref{eq:bco_delta}.
  }
  \label{fig:error_term_values}
\end{figure}

As described in \autoref{subsec:reward_shift}, appropriately adjusting the reward shift decreases the error term resulting with a tighter bound on the DPO loss.
In order to empirically show the effect of reward shift on the error term, we record the change in the error term yielded by BCE and BCO as the learning progresses in \autoref{fig:error_term_values}. 
The figure shows that, with our choice of reward shift, BCO achieves smaller error term compared to BCE, where the reward shift $\delta = 0$.

\begin{figure}[t]
    \centering
    \begin{subfigure}{0.48\textwidth}
        \centering
        \includegraphics[width=\textwidth]{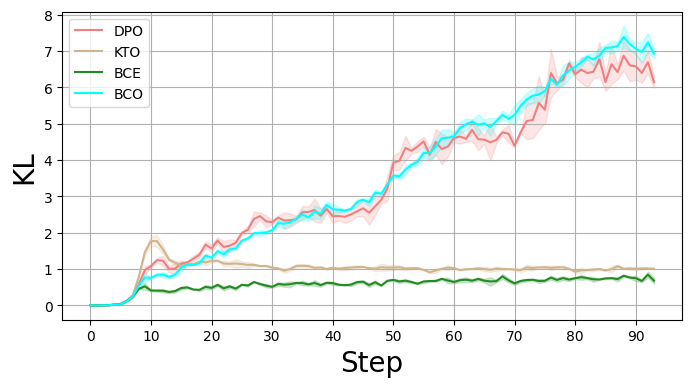}
        \caption{KL by step of alignment methods}
        \label{fig:bco_delta}
    \end{subfigure}
    \begin{subfigure}{0.48\textwidth}
        \centering
        \includegraphics[width=\textwidth]{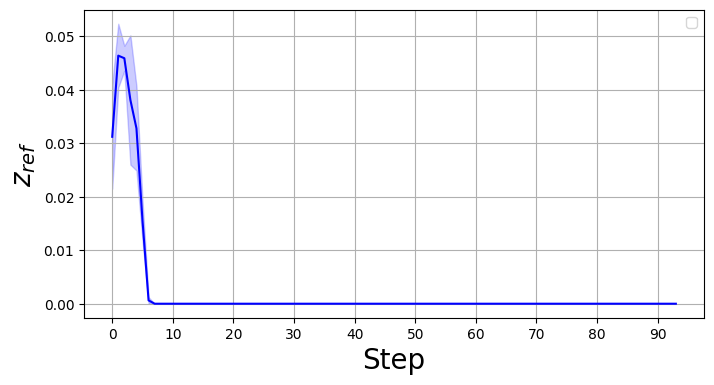}
        \caption{$z_{\text{ref}}$ by step in KTO}
        \label{fig:kto_z_ref}
    \end{subfigure}
    \caption{
        (a) Approximate KL divergence of different algorithms measured using log ratios.
        The plot shows BCO and DPO reaching a relatively high similar KL values while KTO and BCE similarly converging to relatively low KL values. 
        (b) Progress of reference point $z_\text{ref}$ in KTO training.
        The values are taken from \texttt{Llama-3.1-8B} training on Capybara dataset.
        We observed that $z_{\text{ref}}$ consistently collapses to zero. 
    }
    \label{fig:baseline_plots}
\end{figure}

We also compare the effect of reward shift on the KL divergence between the resulting models and the reference model.
Using the relationship between the expected log ratio and the KL divergence \citep{rafailov2024from}, we plot $\text{KL} (\pi_{\text{ref}} \Vert \pi_\theta)$ of BCO and BCE in \autoref{fig:bco_delta}.
The figure shows that while BCE converges at a relatively small KL divergence, BCO is able to match the KL divergence reached by DPO.

Gathering the empirical observations, we conjecture that appropriate reward shift minimizes the error term and the resulting model further assimilates that of DPO.
On the other hand, in the absence of the reward shift, the model converges to a point much closer to the reference model. 
The performance relative to the KL divergence is then conveyed by the significant performance gap between BCE and BCO in \autoref{fig:pref_ultrafeedback}.

A similar observation can be made for KTO as well. 
First, we show the behavior of $z_\text{ref}$ of KTO during the learning process in \autoref{fig:kto_z_ref}.
The plot displays $z_\text{ref}$ collapsing to 0 at early stage of the training.
When $z_\text{ref}=0$, as discussed in \autoref{subsec:distinctions}, the only difference between KTO and BCE is the existence of sigmoid term $\sigmoid (r_\theta)$ in their gradients.
This leads a possible connection between KTO and BCE and their similarities in performance shown in \autoref{fig:pref_combined} and \autoref{fig:likert_helpsteer}.

Additionally, we measure the average length of generated completions for each method.
As detailed in \autoref{tab:token_length_analysis}, we observe a consistent pattern where the generated token lengths for DPO and BCO are similar to each other, while KTO and BCE also exhibit comparable token lengths.

\begin{table}[h]
    \centering
    \begin{tabular}{|c|c|c|c|}
        \hline
         &  & UltraFeedback & Capybara \\ \hline
        \multirow{4}{*}{Llama} & DPO & 326.3 $\pm$ 0.7 & 387.3 $\pm$ 10.7 \\ 
                                 & KTO & 298.2 $\pm$ 3.5 & 322.6 $\pm$ 5.30 \\ 
                                 & BCE & 282.6 $\pm$ 4.5 & 324.1 $\pm$ 5.40 \\ 
                                 & BCO & 333.1 $\pm$ 4.0 & 417.3 $\pm$ 15.5 \\ \hline
        \multirow{4}{*}{Qwen} & DPO & 325.1 $\pm$ 2.2 & 375.2 $\pm$ 4.50 \\ 
                                 & KTO & 317.1 $\pm$ 2.8 & 344.9 $\pm$ 5.00 \\ 
                                 & BCE & 310.2 $\pm$ 3.5 & 330.5 $\pm$ 9.00 \\ 
                                 & BCO & 336.8 $\pm$ 2.4 & 398.0 $\pm$ 10.8 \\ \hline
    \end{tabular}
    \caption{
      Token lengths of generated completions of \texttt{Llama-3.2-3B} and \texttt{Qwen2.5-3B} on UltraFeedback and Capybara datasets.
      Mean and standard deviations are shown.
      The number of generated tokens is approximately proportional to the performance of the model, as illustrated in \autoref{fig:pref_combined}.
      The generated token lengths for DPO and BCO are similar to each other, while KTO and BCE also exhibit comparable token lengths.
    }
    \label{tab:token_length_analysis}
\end{table}

\section{Conclusion}

This paper presents a theoretical foundation for aligning Large Language Models (LLMs) using readily available binary feedback, such as "thumbs-up" or "thumbs-down".
We demonstrate that training a binary classifier implicitly minimizes the Direct Preference Optimization (DPO) loss by mapping desirable outputs to positive labels and undesirable outputs to negative labels.
The binary cross-entropy (BCE) loss used in classifier training acts as an upper bound for minimizing DPO loss, and our proposed reward shift technique further reduces this discrepancy, resulting in stronger alignment.
Our theoretical analyses connects DPO and alignment from binary signal and reveals KTO's potential flaw in choosing a reference point.

Building on these insights, we introduce Binary Classifier Optimization (BCO) as a novel framework for aligning LLMs using binary feedback.
BCO's efficacy is validated through empirical results on paired preference datasets and real-world Likert-5 scale annotation datasets.
Our experiments demonstrate that BCO outperforms KTO and performs competitively with DPO on paired preference datasets.
Notably, on real-world data, BCO consistently surpasses both DPO and KTO across various LLM configurations, including Qwen and Llama, showcasing its robustness and applicability.
This binary classifier perspective on alignment offers a potential complement to preference-based alignment and could contribute to a deeper understanding of multi-stage preference tuning, paving the way for future advancements in AI alignment.

\section{Limitation}

The primary limitation of this research is the absence of real-world benchmarks utilizing binary annotations.
Practical evaluations, essential for demonstrating the utility of the proposed approach in real-world applications, are therefore limited.
Although binary feedback collection is easier and more natural compared to gathering pairwise preference data, particularly in real-world services such as ChatGPT or Claude, the lack of such benchmarks restricts the thoroughness of our evaluations.

Second, this research direction is still under development, with relatively few algorithms proposed to address the challenges in this field.
Consequently, it is difficult to conduct comprehensive analyses across different approaches, further limiting the scope of evaluation.

From an algorithmic perspective, the proposed method focuses on optimizing the upper bound of the Direct Preference Optimization (DPO) loss function which introduces a gap between the optimized upper bound and the actual DPO loss.
Minimizing an upper bound does not always equate to minimizing the original objective function, potentially leading to unintended effects on the model’s generalization and robustness.
Further investigation is needed to understand the impact of this gap on practical model performance.

Lastly, the algorithm relies on binary feedback, limiting its ability to fully utilize the rich comparative information available in preference datasets.
Preference data offers nuanced insights through pairwise comparisons, but the algorithm only captures binary positive/negative signals, leading to incomplete utilization of available information.
This limitation could result in suboptimal performance in tasks aimed at optimizing preference datasets.

\section*{Acknowledgments}

We thank Jiyeon Ham, Changmin Lee, Daejin Jo and Hyunwoong Ko for helpful and constructive feedback.

\bibliography{custom}

\clearpage

\appendix



\section{Proofs} \label{sec:proofs}

\subsection{The log of sigmoid of a sum exceeds the sum of the logs of the sigmoids} \label{subsec:sigmoid_lemma_proof}

\newtheorem*{lemma*}{Lemma}
\newtheorem*{theorem*}{Theorem}

\begin{lemma*}
The log of sigmoid of a sum exceeds the sum of the logs of the sigmoids. i.e. $\log \sigmoid (x + y) > \log \sigmoid(x) + \log \sigmoid(y)$ for all $x, y \in \mathbb{R}$ 
\end{lemma*}

\begin{proof}
\begin{align}
&\log \sigmoid (x + y)  = - \log \left( 1 + e^{- (x + y)} \right) \nonumber \\
&\log \sigmoid(x) + \log \sigmoid(y) \nonumber \\
&= - \log (1 + e^{-x} ) - \log (1 + e^{-y}) \nonumber \\
&= - \log \left( (1 + e^{-x})(1 + e^{-y}) \right) \nonumber \\
&= - \log (1 + e^{-(x + y)} + e^{-x} + e^{-y}) \label{eq:upper_bound}
\end{align}
As $e^{-x}$ and $e^{-y}$ are both greater than 0, the proposition holds.
\end{proof}

\subsection{BCE loss is the upper bound of DPO loss even under constant reward shift} \label{subsec:upper_bound_proof}

\begin{theorem*} 
Binary cross entropy is an upper bound of Direct Preference Optimization loss even if the reward is shifted by a constant $\delta$. i.e.
\begin{align*}
&\E_{(x, y_w, y_l) \sim \train} [-\log \sigmoid \left( r_\theta (x, y_w) - r_\theta(x, y_l) \right) ] \\
&< \E_{(x, y_w, y_l) \sim \train} [- \log \sigmoid(r_\theta(x, y_w) - \delta) \\
& \qquad - \log \sigmoid(- (r_\theta(x, y_l) - \delta))]
\end{align*}
\end{theorem*}

\begin{proof}
\begin{align*}
&\E_{\train}
\left[ - \log \sigmoid (r_\theta (x, y_w) - r_\theta (x, y_l)) \right] \\
&= \E_{\train}
\left[ - \log \sigmoid ((r_\theta (x, y_w) - \delta) - (r_\theta (x, y_l) - \delta)) \right] \\
&< \E_{\train}
[ - \log \sigmoid (r_\theta (x, y_w) - \delta) \\
&\qquad\qquad - \log \sigmoid (- (r_\theta(x, y_l) - \delta))]
\end{align*}
\end{proof}

\subsection{Optimal $\delta$ to minimizing the error term} \label{subsec:optimal_delta_proof}

\begin{theorem*}
The minimum of the error term $e^{-(r_\theta(x, y_w) - \delta)} + e^{r_\theta(x, y_l) - \delta}$ can be achieved when $\delta = (r_\theta(x, y_w) + r_\theta(x, y_l)) / 2$
\end{theorem*}

\begin{proof}
Due to AM-GM inequality, 
\begin{equation*}
e^{-(r_\theta(x, y_w) - \delta)} + e^{r_\theta(x, y_l) - \delta} \geq 2 \sqrt{e^{r_\theta(x, y_l) - r_\theta(x, y_w)}}
\end{equation*}
and the minimum is achieved if and only if $e^{-(r_\theta(x, y_w) - \delta)} = e^{r_\theta(x, y_l) - \delta}$.

If we take the logarithm of both sides and appropriately rearrange the equation, we get $\delta = (r_\theta(x,A y_w) + r_\theta(x, y_l)) / 2$.
\end{proof}

\section{Average implicit reward is proportional to negative KL} \label{sec:average_implicit_reward}

In this section, we replicate \citet{rafailov2024from}'s analysis of average implicit reward for self-completeness.

Expanding $\text{KL}(\pi_{\text{ref}} (\cdot \mid x) \Vert \pi_\theta (\cdot \mid x))$, we get expected implicit reward of a policy under the reference model. i.e.
\begin{align}
&- \beta \text{KL}(\pi_{\text{ref}} (\cdot \mid x) \Vert \pi_\theta (\cdot \mid x)) \nonumber \\
&\qquad = \mathbb{E}_{y \sim \pi_{\text{ref}} (\cdot \mid x)} \left[ \beta \log \frac{\pi_\theta (y \mid x)}{\pi_{\text{ref}} (y \mid x)} \right] \label{eq:average_implicit_reward}
\end{align}

if we run SFT on the preference dataset $\train$, which is common practice recommended by \citet{rafailov2023direct}, \autoref{eq:average_implicit_reward} is approximately equivalent to
\begin{equation*}
\frac{1}{2} \mathbb{E}_{y_w, y_l \sim \train} \left[ \beta \log \frac{\pi_\theta (y_w \mid x)}{\pi_{\text{ref}} (y_w \mid x)} + \beta \log \frac{\pi_\theta (y_l \mid x)}{\pi_{\text{ref}} (y_l \mid x)} \right]
\end{equation*}

\section{Implementations} \label{sec:implementation_details}
\label{app:implement}

During the initial supervised fine-tuning (SFT) phase, we trained the model for 3 epochs using a batch size of 128 and a learning rate of $1e-5$.
We set the maximum sequence length to 4096 and employed the AdamW optimizer \citep{Loshchilov2017DecoupledWD} in conjunction with a linear learning rate scheduler.

For the subsequent alignment training using DPO, KTO, BCE, or BCO techniques, we implemented a linear scheduler with a warm-up ratio of 0.1 on both the UltraFeedback and Capybara datasets.
We constrained the maximum token length to 2048, with a maximum prompt length of 1536 and a maximum completion length of 512.
The reward-KL trade-off coefficient $\beta$ was set to 0.1, and we used a learning rate of $5e-7$.
Given the size disparity between the datasets, we trained the models for 1 epoch on UltraFeedback and 4 epochs on Capybara, as the latter is approximately one-quarter the size of the former.

For training on the HelpSteer2 dataset \autoref{fig:likert_helpsteer}, we largely adhered to the methodology outlined by \citep{wang2024helpsteer2}.
Specifically, we trained the models for 7 epochs using a constant learning rate scheduler with a learning rate of $2e-7$.
The conversion of the HelpSteer2 dataset into a preference dataset resulted in an imbalance, as noted in \autoref{tab:stat_helpsteer2}.
To address this imbalance, we set $\lambda_U$ in \autoref{paragraph:kto} to $1.58 \approx \frac{(1 - 0.38)}{0.38}$.
For balancing in BCO, we employed oversampling of the thumbs-up dataset.
This adjustment was necessary to prevent the scale of the expected log-sigmoid rewards for the thumbs-up dataset in \autoref{eq:bco_objective} from being less than that of the thumbs-down dataset, which could lead to unstable training.

For the models presented in \autoref{tab:alignment_benchmark_results}, we conducted training for 3 epochs using a linear learning rate scheduler with a warmup ratio of 0.1. The learning rate was set to $5e-7$.
Throughout all training phases, we utilized mixed precision with bfloat16 to optimize computational efficiency.
Additionally, we implemented FlashAttention-2 \citep{dao2023flashattention2} to further enhance training performance.

\begin{figure*}[h]
    \centering
    \includegraphics[width=0.8\textwidth]{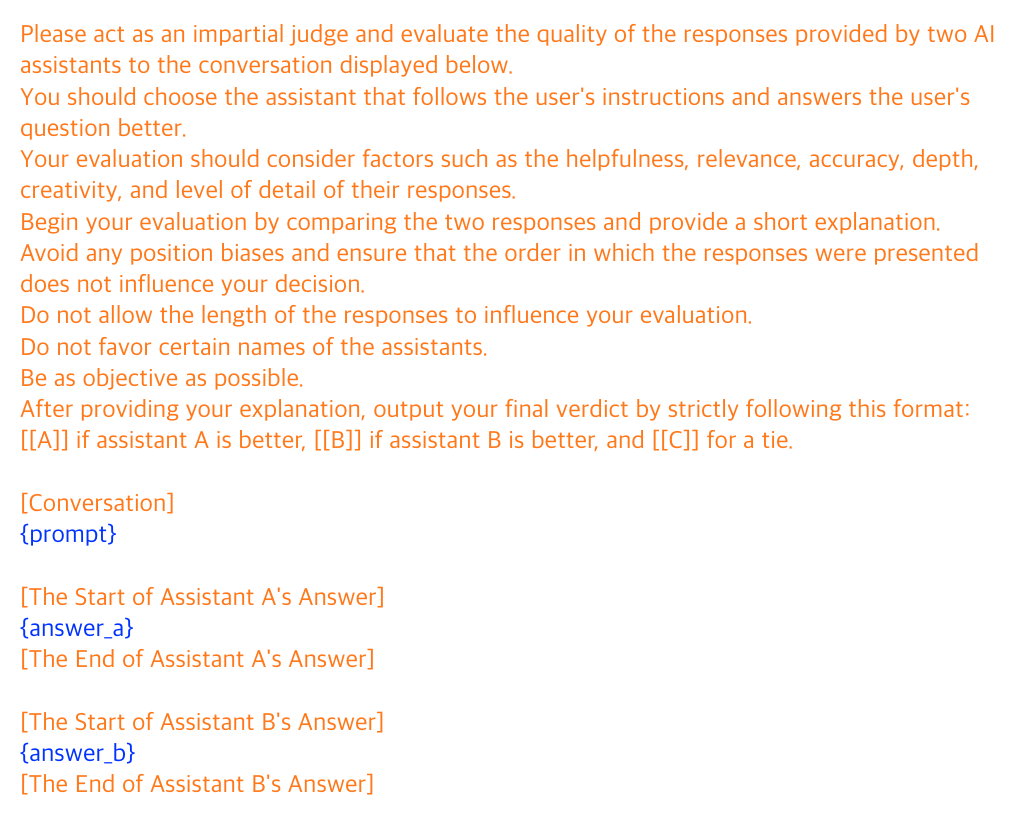}
    \caption{LLM as a judge prompt for UltraFeedback, Capybara, and HelpSteer2 datasets.}
    \label{fig:llm_as_a_judge_prompt}
\end{figure*}

For response generation from each model, we utilize top-p sampling with \( p = 0.95 \) and a temperature parameter of \( 0.7 \).
To measure the win rate using the "LLM as a judge" method, we borrow the judge prompt from \texttt{FastChat} \citep{zheng2023judging}.
See \autoref{fig:llm_as_a_judge_prompt} for "LLM as a judge" prompts.
Across all experiments we adopt \texttt{trl} \citep{vonwerra2022trl} implementations of DPO and KTO.

\section{HelpSteer2 Dataset statistics}

\begin{table}[h]
    \centering
    \begin{tabular}{|c|c|}
    \hline
    \# of preference triplet & 7,221 \\ \hline
    \# of binary triplet & 14,442 \\ \hline
    \% of thumbs-up & 38.65\% \\ \hline
    \end{tabular}
    \caption{
    Statistics of HelpSteer2 dataset after preprocess.
    Preference triplet denotes \{ prompt, chosen, rejected \} triplet and binary triplet denotes \{ prompt, completion, label \} triplet.
    }
    \label{tab:stat_helpsteer2}
\end{table}

In this section, we describe the statistics of HelpSteer2 dataset.
See \autoref{tab:stat_helpsteer2} for the details.

\section{Qualititive Results} \label{sec:qualititive_results}

In this section, we include qualitative experimental results.
We include completions of each method tuned on HelpSteer2; the results are shown in \autoref{tab:completions_after_tune}.
The completions are written after the prompt "\textit{can I adjust the widgets layout in the wp-admin}"

\begin{onecolumn}
\begin{longtable}{lp{0.8\textwidth}}
\toprule
\textbf{Method} & \textbf{Completion} \\
\midrule
\endfirsthead

\multicolumn{2}{c}{\tablename\ \thetable\ -- \textit{Continued from previous page}} \\
\toprule
\textbf{Method} & \textbf{Completion} \\
\midrule
\endhead

\midrule
\multicolumn{2}{r}{\textit{Continued on next page}} \\
\endfoot

\bottomrule
\caption{Completions of the prompt "\textit{can I adjust the widgets layout in the wp-admin}"}
\label{tab:completions_after_tune} \\
\endlastfoot

DPO & Yes, you can adjust the layout of widgets in the WordPress administration area (wp-admin).
Widgets are a feature of WordPress that allow you to add various types of content and functionality to your website\'s sidebar or other widget-ready areas.\newline\newline
To adjust the layout of widgets, you can follow these steps:\newline\newline
1. Log in to your WordPress administration area (wp-admin).\newline
2. Navigate to the "Appearance" menu and select "Widgets."\newline
3. You will see a list of available widget areas, such as the sidebar or footer, and a list of widgets that you can add to each area.\newline
4. To adjust the layout of the widgets, you can drag and drop the widgets into different positions within the widget area.
You can also add new widgets or remove existing ones as needed.\newline
5. Once you have arranged the widgets to your liking, click the "Save" button to apply the changes to your website.\newline\newline
It\'s important to note that the exact layout of widgets may vary depending on the theme you are using and the specific widget areas that are available.
Additionally, some widgets may have their own settings that you can adjust to customize their appearance and functionality. \\

KTO & Yes, you can adjust the layout of widgets in the WordPress admin area.
Here are a few tips:\newline\newline
1. Customize the sidebar: You can customize the layout of your sidebar by dragging and dropping widgets into different positions.
To do this, go to the Appearance section of the WordPress admin and select Widgets.
From there, you can drag and drop widgets into the sidebar areas that you want them to appear in.\newline\newline
2. Use a widget-ready theme: Some themes are designed to be widget-ready, which means that they have pre-defined areas where you can add widgets.
These areas are usually located in the sidebar, but they can also be in other areas of the page, such as the footer or the header.
If your theme is widget-ready, you can easily add and remove widgets to customize the layout of your site.\newline\newline
3. Use a page builder: If you want more control over the layout of your site, you can use a page builder plugin.
Page builders allow you to create custom layouts for your pages and posts, and you can add widgets to these layouts.
Some popular page builders include Elementor, Beaver Builder, and Visual Composer.\newline\newline
4. Use a custom plugin: If you want to add widgets to specific areas of your site that aren't widget-ready, you can use a custom plugin.
There are many plugins available that allow you to add widgets to custom areas, such as the header, footer, or anywhere else on your site.\newline\newline
Overall, adjusting the layout of widgets in the WordPress admin area is a simple process that can help you customize the look and feel of your site. \\

BCO & Yes, it is possible to adjust the layout of widgets in the WordPress administration (wp-admin) area.
Here are the steps to do it:\newline\newline
1. Log in to your WordPress website administration area.\newline
2. Click on the "Appearance" menu on the left-hand side of the page, and then click on "Widgets."\newline
3. You will see a list of widget areas, such as the sidebar, footer, and more.
Click on the widget area where you want to adjust the layout.\newline
4. You will see a list of widgets that you can add to the selected widget area.
Drag and drop the widgets to rearrange them in the order you want them to appear on your website.\newline
5. You can also add new widgets by clicking on the "Add a Widget" button, and remove widgets by clicking on the "Delete" button.\newline
6. To change the layout of a specific widget, click on the widget to open its settings.
Here, you can adjust the title, content, and other settings for the widget.\newline
7. Some widgets may have additional layout options, such as the ability to change the number of columns or the size of the widget.
These options will vary depending on the widget you are using.\newline\newline
Once you have adjusted the layout of your widgets, click on the "Save" button to save your changes.
Your website will automatically update with the new widget layout.\newline\newline
It\'s important to note that the exact process for adjusting widget layout may vary slightly depending on the theme you are using and the specific widgets you have installed on your website.
However, the general steps should be similar to the ones outlined above.
If you encounter any issues, you may want to consult your theme\'s documentation or reach out to the theme developer for support.
Additionally, if you are not comfortable making changes to your website\'s code or design, you may want to consider hiring a professional WordPress developer to assist you with making these changes. \\

\end{longtable}
\end{onecolumn}

\end{document}